\documentclass[runningheads]{llncs}
\usepackage[numbers,sort&compress]{natbib}
\usepackage{subcaption}
\usepackage{placeins}
\usepackage{float}
\usepackage{fix-cm}
\usepackage{etex}

\usepackage{dblfloatfix}

\usepackage{nag}

\makeatletter
\@ifpackageloaded{xcolor}{}{%
\usepackage[table,x11names,dvipsnames,svgnames]{xcolor}%
}
\makeatother

\usepackage{colortbl}

\usepackage{graphicx}
\graphicspath{{./figures/}}

\usepackage{wrapfig}

\definecolorset{RGB}{lyft}{}{Red,194,39,36;Sunset,202,53,33;Orange,205,68,20;Amber,200,117,42;Yellow,242,169,52;Citron,186,188,44;Lime,112,159,33;Green,56,139,31;Mint,45,118,56;Teal,52,133,135;Cyan,60,132,202;Blue,55,94,248;Indigo,64,13,247;Purple,115,42,248;Pink,176,25,145;Rose,176,32,75}

\definecolorset{HTML}{h}{}{grapefruit,ED5565;bittersweet,FC6E51;sunflower,FFCE54;grass,A0D468;mint,48CFAD;aqua,4FC1E9;bluejeans,5D9CEC;lavender,AC92EC;pinkrose,EC87C0;lightgray,F5F7FA;mediumgray,CCD1D9;darkgray,656D78}

\definecolorset{RGB}{p}{}{ivory,221,221,221;navy,046,037,133;pine,051,117,056;aqua,093,168,153;azure,148,203,236;sand,220,205,125;peach,194,106,119;carnation,159,074,150;mulberry,126,041,084}

\definecolorset{RGB}{o}{}{black,000,000,000;mint,000,158,115;jeans,000,114,178;azure,086,180,233;banana,240,228,066;orange,230,159,000;ember,213,094,000;amaranth,204,121,167}

\usepackage{microtype}

\usepackage[american]{babel}
\usepackage{array}
\usepackage{multirow}
\usepackage{booktabs}
\usepackage{makecell} 

\ifcsname labelindent\endcsname

\fi
\usepackage[inline]{enumitem}

\usepackage{subcaption}

\newenvironment{lenumerate}[2][]
{\begin{enumerate}[label=(#2\arabic*),leftmargin=0.2in,itemindent=0.15in,#1]}
{\end{enumerate}}

\setlist*[enumerate,1]{label={\itshape\arabic*)}}

\makeatletter
\newcommand{\paragraphswithstop}{%
\let\copyparagraph\paragraph%
\renewcommand\paragraph[1]{\copyparagraph{##1.}}%
}
\makeatother

\usepackage[framemethod=tikz]{mdframed}

\makeatletter
\def\namedlabel#1#2{\begingroup
  #2%
  \def\@currentlabel{#2}%
  \phantomsection\label{#1}\endgroup
}
\makeatother

\makeatletter
\def\namedlabelphantom#1#2{\begingroup
  \def\@currentlabel{#2}%
  \phantomsection\label{#1}\endgroup
}
\makeatother

\newcommand{\parunskip}{\bgroup\unskip\parfillskip=0pt \par\egroup}

\input{preamble/math}
\newcommand{\real}[1]{\mathbb{R}^{#1}{}}

\newcommand{\bmat}[1]{\begin{bmatrix}#1\end{bmatrix}}

\newcommand{\defeq}{\doteq}

\DeclarePairedDelimiter{\norm}{\lVert}{\rVert}

\newcommand{\subjectto}{\textrm{subject to}\;}

\providecommand{\cC}{\mathcal{C}}

\providecommand{\cU}{\mathcal{U}}

\usepackage{units}

  \newcommand{\newcolorlabel}[2]{%
  \expandafter\newcommand\csname #1\endcsname[1]{%
    \tikz[baseline]{\node[text=white,fill=#2,anchor=base,text height=1.3ex,text depth=0.1ex,font=\sffamily\bfseries]{##1}}}%
}

\newcommand{\newcommenter}[2]{%
  \expandafter\newcommand\csname #1\endcsname[1]{%
    \fcolorbox{#2}{#2}{\color{white}\textsf{\textbf{#1}}}
    {\color{#2}##1}}%
  \expandafter\newcommand\csname #1p\endcsname[1]{%
    \pdfcomment[color=#2,voffset=1Em]{#1:##1}}%
  \expandafter\newcommand\csname at#1\endcsname{%
    \fcolorbox{#2}{#2}{\color{white}\textsf{\textbf{@#1}}}
    {\color{#2}}}%
  \expandafter\newcommand\csname #1cite\endcsname[1]{%
    \csname #1\endcsname {[##1]}
  }%
  \expandafter\newcommand\csname #1ref\endcsname[1]{%
    \csname #1\endcsname {$\blacktriangleright$##1}
  }%
  \expandafter\newcommand\csname #1hl\endcsname[2]{%
    \colorbox{#2}{\color{white}\textsf{\textbf{#1}}}\sethlcolor{Azure2}\hl{##2}~%
    \expandafter\ifx\csname commentarrow\endcsname\relax$\leftarrow$\else \commentarrow[#2]\fi~%
    {\color{#2}##1}}%
  \expandafter\newcommand\csname #1hlp\endcsname[2]{%
    \sethlcolor{#2!50!white}\hl{##2}%
    \pdfmargincomment[color=#2,voffset=1Em,icon=Note,open=true,author=#1]{#1:##1}}%
  \expandafter\newcommand\csname #1st\endcsname[2]{%
    \colorbox{#2}{\color{white}\textsf{\textbf{#1}}}\sout{##2}~%
    \expandafter\ifx\csname commentarrow\endcsname\relax$\leftarrow$\else \commentarrow[#2]\fi~%
    {\color{#2}##1}}%
}
\newcommenter{TODO}{DodgerBlue1}
\newcommenter{rtron}{OliveDrab2}

\usepackage{comment}

\usepackage{pdfcomment}

\usepackage{soul}

\usepackage[normalem]{ulem}

\usepackage{csquotes}

\usepackage{listings}
\usepackage{suffix}

\usepackage{environ}

\makeatletter
\newsavebox{\boxifnotempty}
\newcommand{\displayifnotempty}[3]{\sbox\boxifnotempty{#2}\setbox0=\hbox{\usebox{\boxifnotempty}\unskip}%
  \ifdim\wd0=0pt
  \else
  #1\usebox{\boxifnotempty}#3%
  \fi%
}

\newcommand{\ifempty}[2]{\setbox0=\hbox{#1\unskip}%
  \ifdim\wd0=0pt%
  #2%
  \fi%
}

\newcommand{\ifnotempty}[2]{\setbox0=\hbox{#1\unskip}%
  \ifdim\wd0>0pt%
  #2%
  \fi%
}
\makeatother

\newcommand{\switchifempty}[3]{\sbox\boxifnotempty{#1}\setbox0=\hbox{\usebox{\boxifnotempty}\unskip}%
  \ifdim\wd0=0pt{}%
  #2%
  \else{}%
  #3%
  \usebox{\boxifnotempty}%
  \fi%
}

\makeatletter
\@ifundefined{chapter}{\usepackage{algorithm}}{\usepackage[chapter]{algorithm}}
\makeatother
\usepackage{algorithmicx}
\usepackage{algpseudocode}
\makeatother%

\usepackage{scrlfile}

\makeatletter
\newcommand*\newstoreddef[1]{
  \BeforeClosingMainAux{%
    \immediate\write\@auxout{%
      \string\restoredef{#1}{\csname #1\endcsname}%
    }%
  }%
}
\newcommand*{\restoredef}[2]{
  \expandafter\gdef\csname stored@#1\endcsname{#2}%
}
\newcommand*{\storeddef}[1]{
  \@ifundefined{stored@#1}{0}{\csname stored@#1\endcsname}%
}
\makeatother

\usepackage{pageslts}
\NewEnviron{tee}{\BODY\typeout{Marker Tee [start] ^^J \BODY ^^JMaker Tee [end]}}

\usepackage{cleveref}
\crefname{property}{property}{properties}
\Crefname{property}{Property}{Properties}
\crefname{assumption}{assumption}{assumptions}
\Crefname{assumption}{Assumption}{Assumptions}
\crefname{problem}{problem}{problems}
\Crefname{problem}{Problem}{Problems}
\crefname{fact}{fact}{facts}
\Crefname{fact}{Fact}{Facts}
\crefname{remark}{remark}{remarks}
\Crefname{remark}{Remark}{Remarks}
\usepackage{nameref}
\makeatletter
\@ifundefined{corollary}{}{
  \crefname{corollary}{corollary}{corollary}
  \Crefname{corollary}{Corollary}{Corollaries}
  \crefname{example}{example}{examples}
  \Crefname{example}{Example}{Examples}
  \crefname{remark}{remark}{remarks}
  \Crefname{remark}{Remark}{Remarks}
  \crefname{property}{property}{properties}
  \Crefname{property}{Property}{Properties}
  \crefname{assumption}{assumption}{assumptions}
  \Crefname{assumption}{Assumption}{Assumptions}
  \crefname{problem}{problem}{problems}
  \Crefname{problem}{Problem}{Problems}
  \crefname{fact}{fact}{facts}
  \Crefname{fact}{Fact}{Facts}
  \crefname{lemma}{lemma}{lemmas}
  \Crefname{lemma}{Lemma}{Lemmas}
  \crefname{proposition}{proposition}{propositions}
  \Crefname{proposition}{Proposition}{Propositions}
}

\usepackage{tikz}
\usetikzlibrary{calc}
\usetikzlibrary{matrix}
\usetikzlibrary{chains,scopes}
\usetikzlibrary{shapes.geometric}
\usetikzlibrary{arrows.meta}
\usetikzlibrary{decorations.markings}
\usetikzlibrary{decorations.pathreplacing}
\usetikzlibrary{backgrounds}

\tikzset{
  dim above/.style={to path={\pgfextra{
        \pgfinterruptpath
        \draw[>=latex,|->|] let
        \p1=($(\tikztostart)!1.5em!90:(\tikztotarget)$),
        \p2=($(\tikztotarget)!1.5em!-90:(\tikztostart)$)
        in(\p1) -- (\p2) node[pos=.5,sloped,above]{#1};
        \endpgfinterruptpath
      }
    }
  },
  dim double above/.style={to path={\pgfextra{
        \pgfinterruptpath
        \draw[>=latex,|->|] let
        \p1=($(\tikztostart)!3em!90:(\tikztotarget)$),
        \p2=($(\tikztotarget)!3em!-90:(\tikztostart)$)
        in(\p1) -- (\p2) node[pos=.5,sloped,above]{#1};
        \endpgfinterruptpath
      }
    }
  },
  dim below/.style={to path={\pgfextra{
        \pgfinterruptpath
        \draw[>=latex,|->|] let
        \p1=($(\tikztostart)!-1em!-90:(\tikztotarget)$),
        \p2=($(\tikztotarget)!-1em!90:(\tikztostart)$)
        in (\p1) -- (\p2) node[pos=.5,sloped,below]{#1};
        \endpgfinterruptpath
      }
    }
  },
}

\tikzset{
    right angle quadrant/.code={
        \pgfmathsetmacro\quadranta{{1,1,-1,-1}[#1-1]}     
        \pgfmathsetmacro\quadrantb{{1,-1,-1,1}[#1-1]}},
    right angle quadrant=1, 
    right angle length/.code={\def\rightanglelength{#1}},   
    right angle length=2ex, 
    right angle symbol/.style n args={3}{
        insert path={
            let \p0 = ($(#1)!(#3)!(#2)$) in     
                let \p1 = ($(\p0)!\quadranta*\rightanglelength!(#3)$), 
                \p2 = ($(\p0)!\quadrantb*\rightanglelength!(#2)$) in 
                let \p3 = ($(\p1)+(\p2)-(\p0)$) in  
            (\p1) -- (\p3) -- (\p2)
        }
    }
}

\newcommand{\pgfextractangle}[3]{%
    \pgfmathanglebetweenpoints{\pgfpointanchor{#2}{center}}
                              {\pgfpointanchor{#3}{center}}
    \global\let#1\pgfmathresult
}

\usetikzlibrary{shapes.arrows}
\newcommand{\commentarrow}[1][Azure4]{\tikz[baseline=-3pt]{\node[shape border uses incircle, fill=#1,rotate=180,single arrow, inner sep=1pt, minimum size=6pt, single arrow head extend=2pt]{};}}

\tikzset{ax/.style={-latex,line width=2pt}}

\tikzset{camera/.style={fill=Sienna1,fill opacity=0.5},%
image plane/.style={draw=RoyalBlue3,line width=2pt}}

\usepackage{algpseudocode}
\usepackage[T1]{fontenc}
\usepackage{graphicx}
\begin{document}
\title{Composable Model-Free RL for Navigation with Input-Affine Systems}
%
%
\author{Xinhuan Sang*\orcidID{0009-0002-0340-082X},
  Abdelrahman Abdelgawad\orcidID{0009-0005-1433-0072} \and
  Roberto Tron\orcidID{0000-0002-6676-8595}}
\authorrunning{X. Sang et al.}
%
\institute{Boston University, Boston MA 02215, USA \\
  \email{leosang@bu.edu},
  \email{aaoaa@bu.edu},
  \email{tron@bu.edu}}
\maketitle              
\begin{abstract}
  As autonomous robots are increasingly deployed in real-world environments, they must learn to navigate safely in real time. However, anticipating and learning all possible system behaviors is infeasible.
  To address this, we propose a composable, model-free Reinforcement Learning method that
  \begin{enumerate*}
  \item learns a value function and an optimal policy for each individual element in the system; and
  \item combines them to achieve collision avoidance and goal-directed behavior.
  \end{enumerate*}
  The lynchpin of our approach is that the system and environment are unknown and nonlinear, but are assumed to evolve in \emph{continuous time} and be \emph{input-affine} (many electromechanical systems of practical interest can be considered in this category). Based on these assumptions, we make three contributions.
  First, we derive the Hamilton-Jacobi-Bellman (HJB) equation for the \emph{value function} $V$, and show that the corresponding \emph{advantage function $A$ is a quadratic function} of the input actions $u$ and optimal policy $u^*$.
  Second, we propose a novel model-free actor-critic method to \emph{learn static or moving obstacles via optimal policies and value functions}, which uses gradient descent, and is based on the quadratic form of the advantage function mentioned above.
  Third, we show that our models for individual reach and avoid tasks can be \emph{composed online through Quadratically Constrained Quadratic Programs (QCQP)}, with formal guarantees on the avoidance of obstacles (modeled as level sets of the individual value functions), providing a model-free alternative to the popular model-based Control Lyapunov and Barrier Functions methods from the nonlinear control literature.
  We show the effectiveness of our method in simulations compared with a baseline obtained by applying Proximal Policy Optimization (PPO) to a discrete-time approximation of the system.
  \keywords{Collision Avoidance  \and Reinforcement Learning \and Modular Learning.}
\end{abstract}
\section{Introduction}

Robots are increasingly transitioning from controlled environments such as laboratories, factories, and warehouses into complex, dynamic conditions of daily life. However, it is infeasible to anticipate and train robots to handle every possible combination of goals and obstacle configurations that they may encounter in such settings. This makes it a pressing challenge in robotics to develop control methods that enable robots to operate safely and collision-free, while still achieving their goals in such environments.

Reinforcement Learning (RL) offers a promising data-driven approach to address this challenge, allowing an agent to learn value functions and optimal policies through direct interaction with the system~\cite{sutton1998reinforcement}.
In many cases, model-free RL has demonstrated the ability to learn complex behaviors that would be difficult to design manually without prior knowledge of system dynamics or environmental models~\cite{kober2013reinforcement}.
RL has been successfully applied to control problems with relatively simple dynamics but complex free configuration spaces (e.g., multi-UAV path planning when facing multiple threats~\cite{zhang2015geometric} and motion planning in an environment with multiple, possibly moving obstacles~\cite{park2007path}).
However, these successes have focused on discrete state-space formulations, and the learned policies typically apply only to specific task-environment configurations. When the environment or task changes significantly, retraining is usually required.

On the other hand, control theory, particularly linear and nonlinear control, provides provably safe feedback controllers with strong theoretical guarantees. However, these methods typically require accurate system models. An effective class of such controllers relies on Control Lyapunov Functions (CLFs) and Control Barrier Functions (CBFs), which measure progress toward goals and safety relative to obstacles, respectively. By solving a Quadratic Program (QP), it is possible to find control inputs that simultaneously ensure convergence and safety~\cite{ames2016control}. The CLF-CBF approach is robust under mild regularity assumptions (e.g., Lipschitz continuity) and naturally scales to complex environments by adding one linear constraint for each obstacle. Although conceptually similar to the traditional potential field methods~\cite{rimon1990exact}, CLF-CBF QPs offer broader applicability and superior empirical performance. Applications include locomotion control for bipedal robots~\cite{nguyen20163d} and integration with sampling-based path planning~\cite{yang2019sampling}. Nevertheless, these methods rely heavily on explicit expressions for system dynamics and the environment, limiting their practical deployment in uncertain situations, or where an analytical model is difficult to obtain from first principles.

This study introduces a new approach that combines the strengths of both paradigms by developing a model-free CLF-CBF-like control method for input-affine systems in continuous time and space.
A core concept in RL is the value function, which estimates the expected discounted cumulative reward or cost starting from a given state following optimal control~\cite{watkins1989learning}.
We link this notion to CLFs and CBFs in control theory and propose learning CLF-CBF-like actor-critic RL models. To support this, we derive a form of the Hamilton-Jacobi-Bellman (HJB) equation specific to input-affine systems, applicable to many electromechanical platforms of practical interest. Our derivation links the HJB residual to the advantage function and shows that it is quadratic in the optimal policy, providing a compact, model-free method to express the terms required in the CBF constraints for the online composition of policies.

As shown in later sections, any function-approximation framework that provides gradient predictions during training can be used with the proposed method. We employ Gaussian Processes (GPs)~\cite{Rasmussen2006Gaussian} as an interpolation method over sparse data to encode continuous, smooth value functions, and optimal policies across the state space. However, our method can also be applied to other learning frameworks (e.g., Neural Networks), provided that certain necessary conditions are met.
During training, we learn individual value functions (critics) and optimal policies (actors) for goals and obstacles separately. At runtime, these learned results are composed via a QP by incorporating them as linear constraints based on the current configuration of goals and obstacles, enabling real-time collision-free navigation.

To summarize, our approach has three main contributions:
\begin{itemize}
\item We derive a novel advantage function definition based on the HJB residual for input-affine systems. Form an actor-critic learning method based on its quadratic nature with input $u$ and optimal policy $u^*$.
\item From a control-theoretic perspective, our method eliminates the need for explicit system and environment models by employing a novel continuous-time model-free actor-critic learning method.
\item From the reinforcement learning side, generalization is enhanced by modularizing the learning process around composable environmental elements. It provides a formal collision-avoidance guarantee (based on the learned model) via CBF-like constraints for each obstacle.
\end{itemize}

\section{Preliminary Concepts}
\label{sec:pre_concepts}

This section provides an overview of the fundamental concepts in control theory and Reinforcement Learning on which our proposed approach is based.

\subsection{Input-Affine System}
\label{sec:Input-affine}

Input-affine systems represent a class of nonlinear systems where the dynamics are linear-affine in the control input~\cite{khalil2002nonlinear}:
\begin{equation}
  \label{eq:Input-Affine_System}
  \dot{x}=F(x,u)=f(x) + g(x)u,
\end{equation}
where $x \in X$ is the state of the system, $u \in U$ is the control input, and the mapping $F(x,u)$ that defines the dynamics is composed of the \emph{drift} vector field $f(x)$, and the \emph{control} vector field $g(x)$.
A key consequence of having a dynamics in the form of Eq.~\ref{eq:Input-Affine_System} is that the derivative (with respect to time) of any function $V$ evaluated along the trajectories of the system, that is, its \emph{Lie derivative}, will be linear-affine with respect to the control inputs $u(t)$, that is,
\begin{equation}
  \label{eq:Lie derivative}
  \dot{V}=\frac{d}{dt} V\bigl( x(t) \bigr)=\nabla_x V^{\intercal} \dot{x}=\nabla_x V^{\intercal} F(x,u)=\nabla_x V^{\intercal} f(x)+\nabla_x V^{\intercal} g(x)u.
\end{equation}
While all the forms in Eq.~\ref{eq:Lie derivative} are mathematically equivalent, in the following we will choose different forms depending on what is most useful in the context.

\subsection{CBF Constraint}
\label{subsec:cbf_constraint}
Given a CBF $H(x):\mathcal{D}\to\mathbb{R}$, define the associated safe set
\begin{equation}
\mathcal{C} \defeq \{x\in\mathcal{D}\mid H(x)\ge 0\}.
\end{equation}
We say that $\cC$ is \emph{forward invariant} for the (closed-loop) dynamics if any trajectory
starting in $\cC$ remains in $\cC$ for all future times (i.e., $x(0)\in\cC
\Rightarrow x(t)\in\cC$ for all $t\ge 0$, over the interval of existence of the solution).
Showing that $\cC$ is forward invariant is equivalent to showing that $H(x(0))\ge 0 \Rightarrow H(x(t))\ge 0$ for all $t\ge 0$; a sufficient condition to achieve this is given by the \emph{CBF constraint}~\cite{ames2016control}
\begin{equation}
\label{eq:CBF constaint}
  \dot{H}+c_h H=\nabla_xH^\intercal\dot{x}+c_h H \geq 0.
\end{equation}
The constant $c_h\geq 0$ controls how conservative the constraint is; for example, $c_h=0$ corresponds to the most conservative case where $H$ is never allowed to decrease. 

\subsection{Advantage Function}
In the field of discrete time RL~\cite{baird1993advantage,sutton1999policy}, the \emph{value function} $V(x)$ represents the total expected reward/cost when starting at state $x$ and following optimal control $u^*(x)$. The \emph{state-action value function} $Q(x,u)$ represents the total expected reward/cost when starting in state $x$ and taking action $u$ and then followed by optimal control $u^*$. The \emph{advantage function} $A(x,u)$ represents the difference
\begin{equation}
\label{eq:advantage_function_def}
    A(x,u) = Q(x,u)-V(x).
\end{equation}
 By definition, we have:
\begin{equation}\label{eq:min A}
  A(x,u^*)=\min_u A(x,u)=\min_u Q(x,u) - V(x)=Q(x,u^*)-V(x)=0.
\end{equation}
\section{Theoretical Derivation}
\subsection{Hamilton–Jacobi–Bellman Equation}
\label{sec:HJB}
The Hamilton–Jacobi–Bellman (HJB) equation is a fundamental result in optimal control theory that provides a necessary condition for optimality in continuous-time, continuous-state decision-making problems~\citep{kirk2004optimal}, formulated as a nonlinear partial differential equation. This section provides an HJB derivation for a time-discounted setting under state-dependent termination conditions.

We consider the trajectories of the system $x(t)$ (episodes) uniquely determined by an input signal $u(t)$ and an initial condition $x(t_0)=x_0$, terminating at time $T(x_0,u)$ according to user-defined state conditions (e.g., hitting an obstacle, reaching the goal, or reaching an out-of-bound region).
Consider a discount factor $\lambda$ and no terminal cost. The value function $V$ is then defined as
\begin{equation}
\label{eq:HJB_discounted}
V(x_0,t_0)=\min_{u\in \cU}\int_{t_0}^{T(x_0,u)} e^{-\lambda (t-t_0)}C(x(t),u(t))dt
\end{equation}
where $\cU$ is the set of right-continuous control signals defined for $t\geq t_0$, $x(t)$ is the solution to Eq.~\ref{eq:Input-Affine_System} uniquely determined by $u$ and $x_0$, $C(x,u)$ is a user-defined infinitesimal cost function.\footnote{The definition and derivations are valid with a terminal cost after minor modifications, but this is not used in this paper.}
Starting from a later state $x_h\dot{=} x(t_0+h)$, the HJB equation is similar:
\begin{equation}
\label{eq:HJB_h}
V(x_h,t_0+h)=\min_{u\in \cU}\int_{t_0+h}^{T(x_0,u)} e^{-\lambda (t-(t_0+h))}C(x(t),u(t))dt
\end{equation}
Breaking the integral, and substituting the second term with Eq.~\ref{eq:HJB_h} we obtain:
\begin{equation}
    V(x_0,t_0) =\min_{u\in \cU}\biggl\{\int_{t_0}^{t_0+h} e^{-\lambda (t-t_0)}C(x(t),u(t)) dt+ e^{-\lambda h} V(x_h,t_0+h)\biggr\}
\end{equation}

Subtracting $V(x_0,t_0)=e^{\lambda 0} V(x_0,t_0)$ on both sides we get:
\begin{equation}
    0=\min_{u\in \cU}\biggl\{\int_{t_0}^{t_0+h} e^{-\lambda (t-t_0)}C(x(t),u(t))dt+ e^{-\lambda h} V(x_h,t_0+h)-e^{\lambda 0} V(x_0,t_0)\biggr\}.
\label{eq:HJB_subtract_V}
\end{equation}

Assuming time-invariant dynamics $\dot{x}=F(x,u)$, dividing both sides by $h$, adding and removing $e^{-\lambda h}V(x_0,t_0)$, and taking the limit $h\rightarrow0$, derivatives appear as follows:
{\small
\begin{align}
\label{eq:HJB}
    0=&\lim_{h\rightarrow0}\frac{1}{h}\min_{u\in\cU}\biggl\{\int_{t_0}^{t_0+h}e^{-\lambda (t-t_0)}C(x(t),u(t))dt+ 
     e^{-\lambda h} V(x_h,t_0+h)-e^{\lambda 0} V(x_0,t_0)\biggr\}\nonumber\\
    =&\lim_{h\rightarrow0}\frac{1}{h}\min_{u\in\cU}\biggl\{\int_{t_0}^{t_0+h}e^{-\lambda (t-t_0)}C(x(t),u(t))dt+\nonumber \\
    &\qquad \qquad  \quad e^{-\lambda h} \bigl( V(x_h,t_0+h)-V(x_0,t_0)\bigr)+\bigl(e^{-\lambda h}-e^{\lambda 0}\bigr) V(x_0,t_0)\biggr\}\nonumber\\
    =&\min_{u\in\cU}\biggl\{e^{-\lambda (0)}C(x_0,u_0) + e^{-\lambda 0}\left.\frac{d}{dt} V(x,t) \right|_{t=t_0}+ \left.\frac{d}{dt} e^{-\lambda t}\right|_{t=0} V(x_0,t_0)\biggr\}
\end{align}}
With time-invariant dynamics, the result of Eq.~\ref{eq:HJB_discounted} depends only on the starting state $x_0$ and not on the initial time $t_0$. Thus, we can substitute the term $\frac{d}{dt} V(x,t)|_{t=t_0}$ with Eq.~\ref{eq:Lie derivative} and drop the dependency on a specific time $t_0$:
\begin{equation}\label{eq:HJB general}
    0=\min_{u}\biggl\{C(x,u)-\lambda V(x)+ \nabla_{x} V(x)^{\intercal} F(x,u) \biggr\}.
\end{equation}
\subsection{Special Form of the Advantage Function}
By analogy between the HJB and the corresponding Bellman equation in discrete time~\cite{baird1993advantage,sutton1999policy} and between Eq.~\ref{eq:HJB general} and Eq.~\ref{eq:min A}
we define the \emph{differential advantage function} $A(x,u)$ as the argument of the $\min$ operator of Eq.~\ref{eq:HJB general}:
\begin{equation}
\label{eq:advantage_function}
    A(x,u) = C(x,u)-\lambda V(x)+ \nabla_{x} V(x)^{\intercal} F(x,u).
\end{equation}
With this definition, the HJB equation becomes a special condition on the advantage function and for the discrete time case in Eq.~\ref{eq:min A}, for the optimal control $u^*$, we have
\begin{equation}\label{eq:advantage_min}
    A(x,u^*)=\min_{u}A(x,u)=0.
\end{equation}

Furthermore, under practical conditions of the unknown dynamics $F$ and user-defined cost $C$, the advantage $A(x,u)$ takes a form that facilitates learning. 
\begin{proposition}
  Assume an unknown input-affine dynamics of form Eq.~\ref{eq:Input-Affine_System}, and a quadratic cost function $C(x,u)$ of form
\begin{equation}
\label{eq:cost_function}
    C(x,u) = \frac{1}{2}(u-u_c(x))^\intercal R(x)(u-u_c(x))+q_c(x),
\end{equation}
where $R(x)$ is positive semidefinite, and $u_c(x)$ is a vector-valued function.
Then the differential advantage function $A(x,u)$ is quadratic in the optimal control $u^*(x)$, specifically:
\begin{equation}
    A(x,u) = \frac{1}{2}\bigl(u-u^*(x)\bigr)^\intercal R(x)\bigl(u-u^*(x)\bigr).
    \label{eq:quadratic_advantage_function}
\end{equation}
\end{proposition}

\begin{proof}
  With the assumptions above and using $f(x)$ and $g(x)$ as placeholders for the unknown dynamics, Eq.~\ref{eq:advantage_function} becomes:
  \begin{equation}\label{eq:advantage_function_input-affine}
    A(x,u)= C(x,u)-\lambda V(x)+ \nabla_x V^{\intercal} f(x)+\nabla_x V^{\intercal} g(x)u.
  \end{equation}

  Because the terms on the right-hand side are either quadratic in $u$ (i.e., $C(x,u)$) or linear in $u$ (i.e., $\nabla_x V^{\intercal} g(x)u$), the advantage function $A(x,u)$ must also be quadratic in $u$, i.e., generically:
  \begin{equation}
    A(x,u) = (u-u_a(x))^\intercal S(x)(u-u_a(x)) +D(x)(u-u_a(x))+E(x).
  \end{equation}
  From Eq.~\ref{eq:advantage_min}, the minimum of $A(x,u)$ is zero and the minimum is achieved at $u=u^*(x)$, hence $E(x),D(x)=0$. Moreover, matching only the quadratic term in Eq.~\ref{eq:advantage_function_input-affine}, which is $C(x,u)$, yields $S(x)=\frac{1}{2}R(x)$. The claim follows.
\end{proof}


\section{Actor-Critic Learning for Single Goals and Obstacles}
\label{sec:actor-critic_learning}
Based on our special form of the advantage function for input-affine systems, we construct our proposed actor-critic algorithm that we will use to model individual goals or obstacles (before composing them through the QCQP). Following the definition of actor-critic algorithms~\cite{sutton1999policy}, we intuitively have an actor represented by the optimal policy $u^*(x)$, and a critic represented by the value function $V(x)$~\cite{konda1999actor}.
For simplicity of exposition, we use estimates of the value function $\hat{V}(x)$, of the gradient of the value function $\nabla_x \hat{V}(x)$, and of the optimal policy $\hat{u}^*(x)$ that are linear in the parameters:
\begin{subequations}
\begin{align}
    \hat{V}(x) &= \mathcal{M}_V(x)\theta_{V},\\
    \nabla_x \hat{V}(x) &= \mathcal{M}_V'(x)\theta_{V},\\
    \hat{u}^*(x) &= \mathcal{M}_u(x)\theta_{u^*},
\end{align}
\label{eq:general_model}
\end{subequations}
where $\theta_{V}$ and $\theta_{u^*}$ are the parameters and data used by the value function model $\mathcal{M}_{V}$ and optimal policy model $\mathcal{M}_{u}$, respectively. Note that the model for obtaining the gradient of the value function, $\mathcal{M}'_{V}$, is based on the same parameters and data as in the model $\mathcal{M}_{V}$.

The parameters $R(x)$, $u_c(x)$, and $q_c(x)$ in the cost function are part of the design of the learning algorithm (i.e., determining how the instantaneous cost is determined). In this case, we chose $R(x)=I$, $q_c(x)=0$, and $u_c(x)=0$. Then we have that the cost function Eq.~\ref{eq:cost_function} simplifies, and we define two different models for the advantage function, one from Eq.~\ref{eq:advantage_function} based on critic $\hat{V}(x)$, and one from Eq.~\ref{eq:quadratic_advantage_function} based on actor $\hat{u}^*(x)$:
\begin{subequations}
\begin{align}
    C(x_i,u)&=\frac{1}{2}u^\intercal u,\\
    \hat{A}_{critic}(x_i,u)&=\frac{1}{2}u^\intercal u +\nabla_x \hat{V}(x_i)^\intercal \dot{x}-\lambda \hat{V}(x_i).\label{eq:A critic}\\
    \hat{A}_{actor}(x_i,u)&=\frac{1}{2} \norm{u-\hat{u}^*(x_i)}^2,
    \end{align}
\end{subequations}
We define loss as the difference between the \emph{critic} and \emph{actor} advantage functions:
\begin{equation}\label{eq:loss}
    Loss=\frac{1}{2}(\hat{A}_{actor}(x,u)-\hat{A}_{critic}(x,u))^2.
\end{equation}

We update $\theta_V$ and $\theta_{u^*}$ using gradient descent:
\begin{subequations}
\begin{align}
    \frac{\partial Loss}{\partial \theta_V}=(\hat{A}_{actor}&-\hat{A}_{critic})(\lambda \mathcal{M}_V(x_i)-\mathcal{M}'_V(x_i)\dot{x}),
    \label{eq:theta_V_gradient}\\
    \frac{\partial Loss}{\partial \theta_{u^*}}=-(\hat{A}_{actor}&-\hat{A}_{critic})\bigl(u-\hat{u}^*(x)\bigr)\mathcal{M}_u(x_i)\label{eq:theta_u_gradient},\\
    \theta_V &\leftarrow \theta_V -\eta \frac{\partial Loss}{\partial \theta_V}\label{eq:value_update_general},\\
    \theta_{u^*} &\leftarrow \theta_{u^*} -\eta \frac{\partial Loss}{\partial \theta_{u^*}}.
\end{align}
\label{eq:general_gradient_descent}
\end{subequations}

\begin{remark}\label{remark:no model}
  While we used the model $f(x),g(x)$ for the derivation, we do not need to know it for computing the gradient updates in Eq.~\ref{eq:general_gradient_descent}. Similarly to how the discrete case assumes that we observe pairs of states $x[k],x[k+1]$, in our continuous-time case we assume that we observe pairs $x(t),\dot{x}(t)$.
\end{remark}
\section{Modular Training of Single Goals and Obstacles}
A central design choice in our framework is to treat a navigation scene not as a single monolithic RL task but as a \emph{composition} of reusable behaviors associated with individual \emph{environment elements}, namely goals and obstacles. This mirrors the way CLF-CBF controllers scale to complex scenes by introducing one constraint per obstacle and combining them online in a QCQP, but without requiring explicit models of the system dynamics or obstacle geometry.
\subsection{Element-wise Training in Local Coordinate}
\label{sec:element-wise_training}
Let $\mathcal{E}$ denote the set of elements present in a scene: one goal element $e_g$ and $M$ obstacle elements $\{e_i\}_{i=1}^M$. For each element $e_j\in\mathcal{E}$, we define a state $x_{j}$, which includes the agent state (e.g., pose, velocity, and other available parameters relative to the element) and the element's parameters (e.g., shape descriptors when available).
The state $x_{j} \in X_{j}$ is only relative to the goal/obstacle element, so the learned functions are not affected across global layouts by construction.

For \emph{each} element $e_j$ we train critic $V_j(x_{j})$ and actor $u_{j}^*(x_{j})$ using the continuous-time HJB residual from Eq.~\ref{eq:advantage_function_input-affine}. The termination set $\mathcal{X}_{{j},term}\subset X_{j}$ is the contact/reach set; we enforce $V_{j}(x)=0$ for $x\in \mathcal{X}_{{j},term}$. Importantly, \emph{goals and obstacles are trained in the same manner}: we always define a \emph{reaching} task whose termination set corresponds to entering the element. For the goal, termination occurs upon reaching the goal set. For an obstacle, termination occurs upon \emph{contact} with the obstacle. Under this convention, $u_{j}^*(x_{j})$ points in the locally optimal \emph{approach-to-contact} direction, and $V_{j}(x_{j})$ measures the discounted cost-to-contact. Therefore, the distinction between the \emph{goal} and \emph{obstacle} arises \emph{only at deployment time} through how these learned functions are incorporated as linear constraints in the online QCQP.

\subsection{Gaussian Process Implementation}
\label{sec:gp_implementation}
We adopted Gaussian Processes (GPs)~\cite{Rasmussen2006Gaussian} as the function-approximation backbone to validate our theoretical development for three main reasons: \begin{enumerate*}
    \item GP regression provides a modeling framework in which predictions are expressed directly in terms of the data through kernel evaluations, making it straightforward to interpret how the dataset shapes the learned function.
    \item GP posteriors define a distribution over functions on a continuous input domain. With commonly used kernels, the posterior mean is continuous (and, depending on kernel regularity, sample paths can be smooth), which aligns naturally with our continuous-state setting.
    \item When the kernel is differentiable, derivatives of a GP are themselves GPs, and function values and derivative values are jointly Gaussian; therefore, gradients can be obtained in a principled and convenient manner from the same GP model.
\end{enumerate*}


In practice, a discrete subset of the state space $X$, marked as $X^\sharp\subset X$, is the state where we store the mean of the optimal policy $\mu_{u^*}$ and the value function $\mu_V$. For state $x_i\in X$, we can obtain a prediction of the optimal policy $u_p^*(x_i)$ and the value function $\hat{V}(x_i)$ as Gaussian Processes:
\begin{subequations}
\begin{align}
    J_{x_i} =& K_{x_i,X^\sharp}K_{X^\sharp,X^\sharp}^{-1},\\
    \hat{u^*}(x_i) =& J_{x_i} \mu_{u^*}\label{eq:predict_u},\\
    \hat{V}(x_i) =& J_{x_i} \mu_{V}\label{eq:predict_V},
\end{align}
\end{subequations}
where $K_{x_i, X^\sharp}$ is the covariance between the prediction state $x_i$ and the data location $X^\sharp$, and $K_{X^\sharp,X^\sharp}$ is the covariance matrix for all possible combinations within the data location $X^\sharp$. The elements of the covariance matrix are defined by the kernel function $k$; in our case, we use the RBF kernel $k(x_i,x_j)$,
\begin{equation}
    K_{i,j}=k(x_i,x_j)\quad x_i,x_j\in X^\sharp
\end{equation}
Furthermore, with a differentiable kernel, function values and partial derivatives are jointly Gaussian. Therefore, we obtain gradients directly from the same GP posterior (we use the posterior mean in our implementation) to compute the Lie derivative of the value function~\cite{Rasmussen2006Gaussian}:
\begin{subequations}
\begin{align}
    J'_{x_i} &=  K'_{x_i,X^\sharp} K_{X^\sharp,X^\sharp}^{-1},\\
    K'_{x_i,X^\sharp}&=\bmat{\frac{\partial}{\partial [x]_1}K(x_i,X^\sharp) & \cdots & \frac{\partial}{\partial [x]_m}K(x_{i},X^\sharp)},\\
    \nabla_x\hat{V}(x_i) &= J'_{x_i} \mu_{V},\label{eq:predict_delta_V}
\end{align}
\end{subequations}
where $\frac{\partial}{\partial [x]_j}$ is the partial derivative of the $j^{\text{th}}$ dimension of state $x\in \mathbb{R}^m$. For $K$, this partial derivative is applied element-wise.

Substituting $\mathcal{M_V}=J_{x_i}$, $\mathcal{M}'_V=J'_{x_i}$, and $\mathcal{M}_u=J_{x_i}$ into Eq.~\ref{eq:general_gradient_descent} yields gradient-descent updates for the stored base-point means $(\mu_V,\mu_{u^*})$.


\subsection{Discrete-time Baseline using Proximal Policy Optimization (PPO)}
To further evaluate our system, we compared it with Proximal Policy Optimization (PPO) \cite{schulman2017proximal}. PPO is an on-policy actor--critic algorithm that alternates between \begin{itemize}
    \item collecting rollouts under the current stochastic policy $u_{\theta}(x)$ and
    \item updating $\theta$ by maximizing a clipped surrogate objective while fitting a critic $V_\phi(x)$ to predict the discounted return \cite{schulman2017proximal}.
\end{itemize}

We use PPO to train a critic and an actor neural network. We use the Stable-Baselines3 implementation of PPO \cite{raffin2021stable}, which performs multiple epochs of minibatch
updates per rollout buffer, including common hyperparameters such as the rollout horizon $n_\text{steps}$, GAE $\lambda_{GAE}$, and clipping range \cite{raffin2021stable}.
The PPO critic $V_\phi(x)$ estimates the expected discounted return $\mathbb{E}[\sum_{k\ge 0}\gamma^k r_k]$. Under our choice $r_k=-\Delta t\,C(x_k,u_k)$ (plus terminal shaping), $V_\phi$ represents a negative
cost-to-go (up to discretization and shaping).  The optimal control action $u^*$ is obtained as the mean action of the policy $\pi$. 

\subsection{Simulation Results for Single Static and Moving Elements}

\paragraph{Simulation setup.}
Here, we show the results of learning individual elements. 
To emphasize that our approach does not require an explicit model of the element's shape or dynamics, we consider navigation tasks in $\real{2}$ in which the learner observes only a \emph{relative position} state, while key dynamical effects remain unobserved. Each environmental element $e_j$ (a goal region or obstacle) induces a local navigation objective. 
To make these behaviors interpretable, after training, we evaluated each actor-critic pair on a dense grid around the element and visualized (i) the learned value $\hat{V}_j(x_j)$ as a contour map and (ii) the learned control $\hat{u}_j^*(x_j)$ as a vector field in the coordinates of the element.

\begin{figure}[htbp]
  \centering
  \begin{subfigure}[b]{0.35\textwidth}
    \centering
    \includegraphics[width=\linewidth]{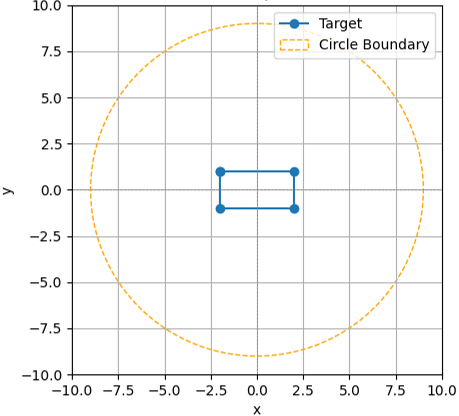}
    \caption{Range and element}
  \end{subfigure}
  \begin{subfigure}[b]{0.35\textwidth}
    \centering
    \includegraphics[width=\linewidth]{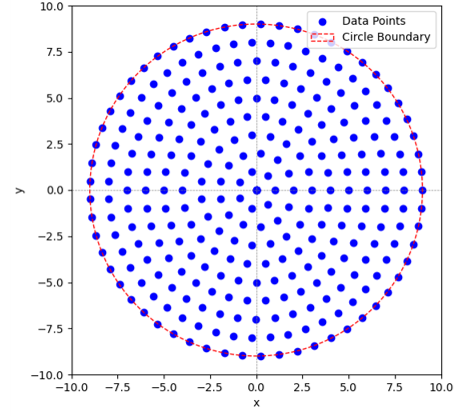}
    \caption{Model data points}
  \end{subfigure}
  \caption{Training range setup for a typical 2-dimensional navigation scenario: (a) shows a circular state space (orange dashed circle) around the element (blue rectangle) and (b) shows the $X^{\sharp}$ used to store $\mu_V$ and $\mu_{u^*}$ as blue dots.}
  \label{fig:Training_Env}
\end{figure}

\begin{figure}[htbp]
  \centering
  \begin{subfigure}[b]{0.35\textwidth}
    \centering
    \includegraphics[width=\linewidth]{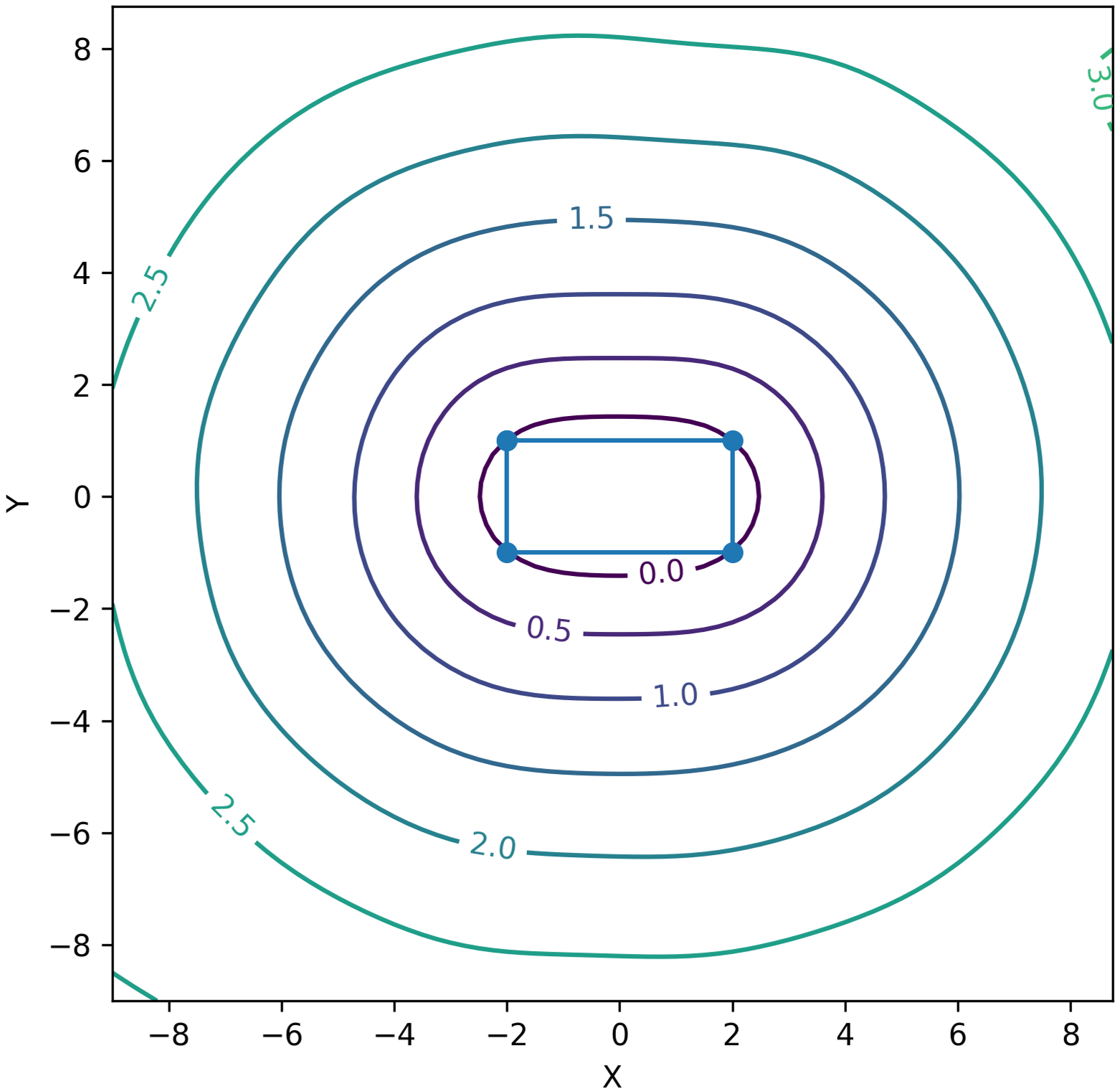}
    \caption{$V(x)$}
  \end{subfigure}
  \begin{subfigure}[b]{0.35\textwidth}
    \centering
    \includegraphics[width=\linewidth]{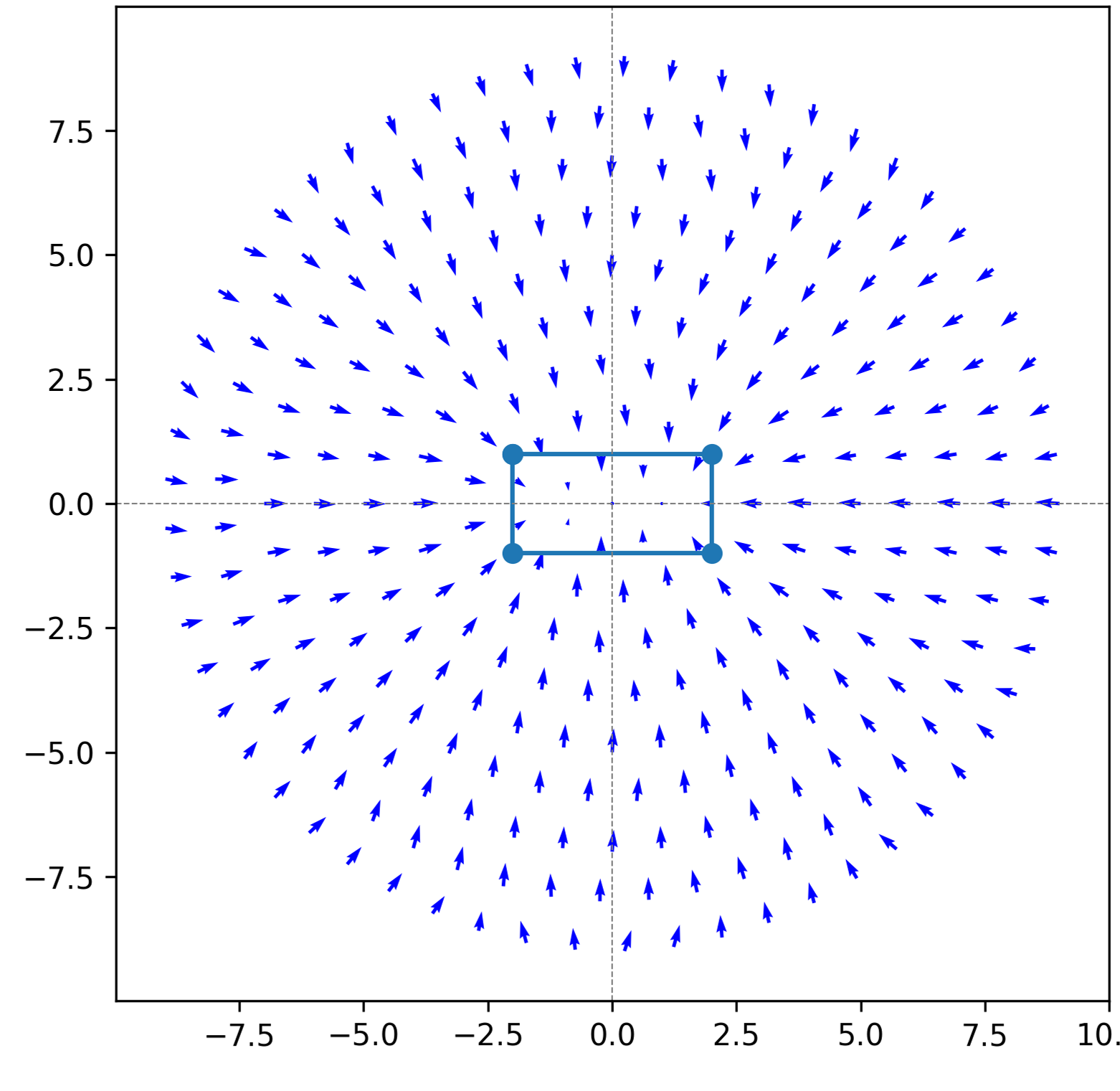}
    \caption{$u^*(x)$}
  \end{subfigure}
  \caption{GP training results for $4\times 2$ stationary rectangle element
  with $20000$ epochs and maximum $100$ steps for each epoch: (a) shows the contours of the learned value function and (b) shows the learned control vector field.}
  \label{fig:Basic_Result}
\end{figure}

\begin{figure}[htbp]
  \centering
  \begin{subfigure}[b]{0.35\textwidth}
    \centering
    \includegraphics[width=\linewidth]{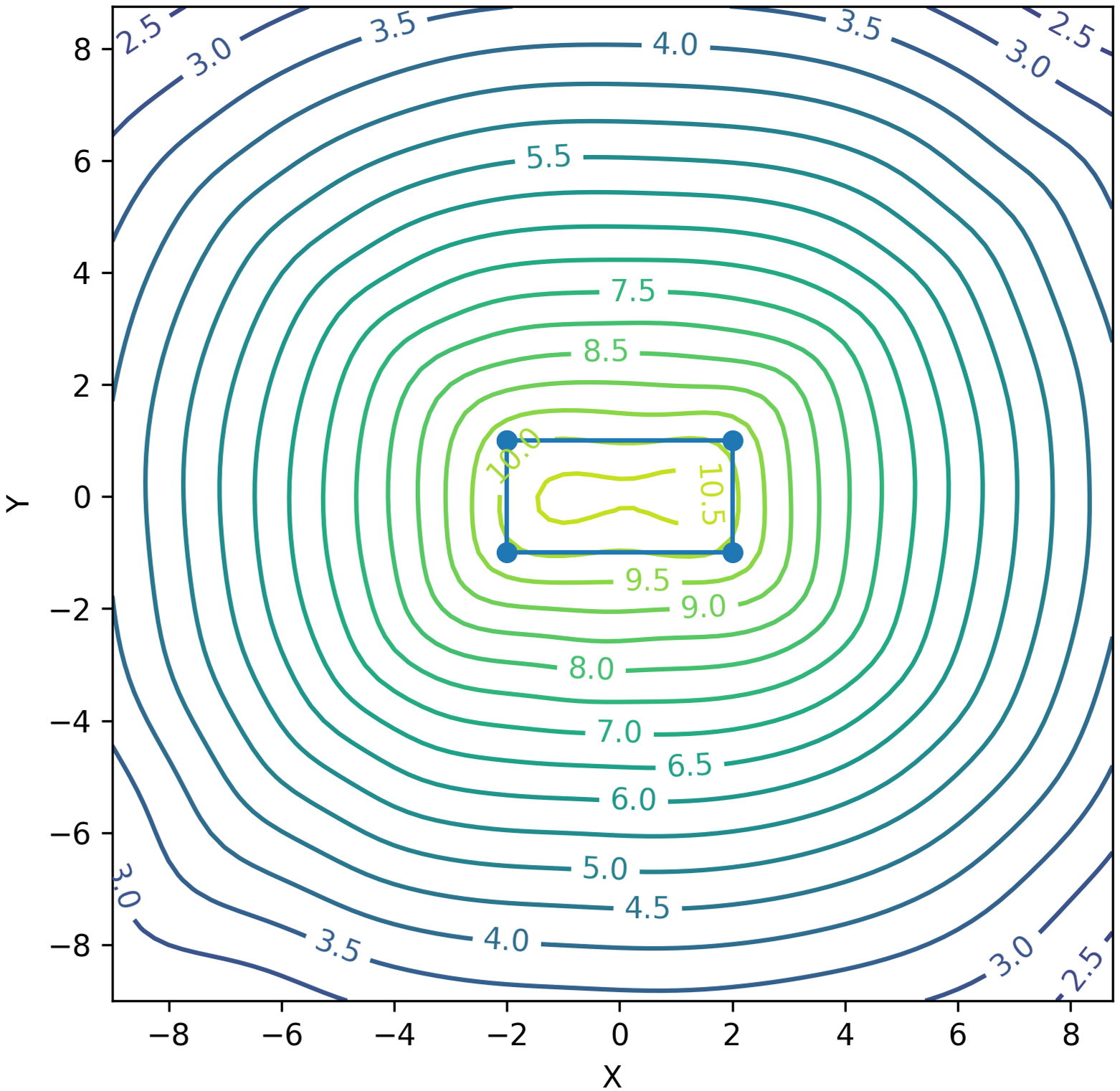}
    \caption{$V_\phi(x)$}
  \end{subfigure}
  \begin{subfigure}[b]{0.35\textwidth}
    \centering
    \includegraphics[width=\linewidth]{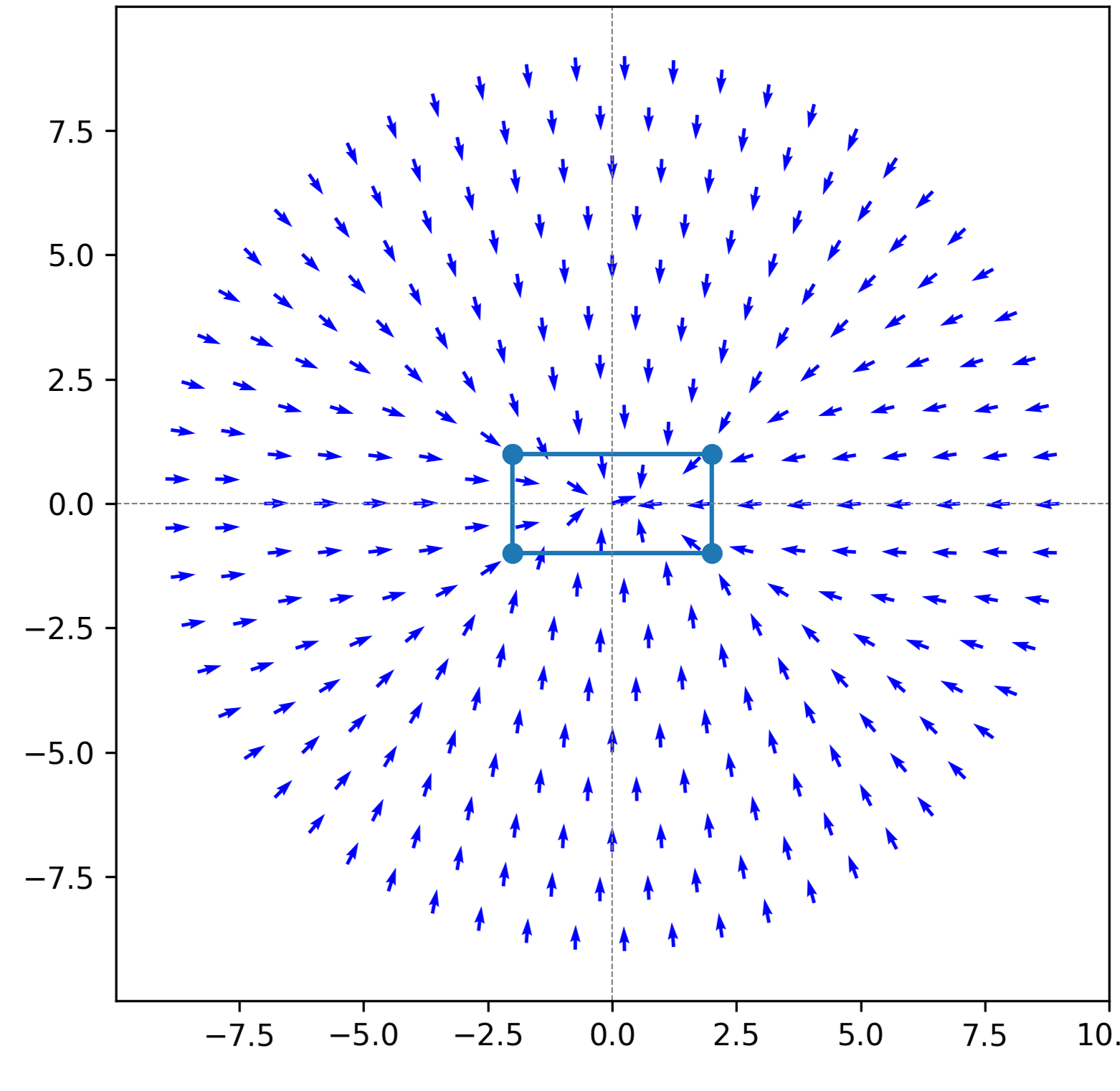}
    \caption{$u_{\theta}(x)$}
  \end{subfigure}
  \caption{PPO training results for $4\times 2$ stationary rectangle element with $20000$ epochs and maximum $100$ steps for each epoch: (a) shows the contours of the learned value function and (b) shows the mean policy control vector field.}
  \label{fig:Basic_Result_PPO}
\end{figure}

\paragraph{Stationary elements.}
Fig~\ref{fig:Training_Env} and~\ref{fig:Basic_Result} show the basic training environment setup and GP results for a 2-dimensional rectangle element. Fig.~\ref{fig:Basic_Result_PPO} shows the baseline PPO training results for the same rectangle as the previous GP results.
The proposed approach can also be easily trained on random polygons, as shown in Fig.~\ref{fig:random_target}.
In Fig.~\ref{fig:Basic_Result} and~\ref{fig:Basic_Result_PPO}, both GP and PPO learn qualitatively similar structures in the value function and control vector field. The primary distinction lies in the sign convention: our method minimizes cost (so $V$ decreases toward contact), whereas PPO maximizes reward (so $V_\phi$ increases toward the obstacle). This agreement confirms that our continuous-time HJB formulation captures the same fundamental structure as standard discrete-time RL, with the PPO critic differing only in sign and scaling.

\begin{figure}[htbp]
  \centering
  \begin{subfigure}[b]{0.35\textwidth}
    \centering
    \includegraphics[width=\linewidth]{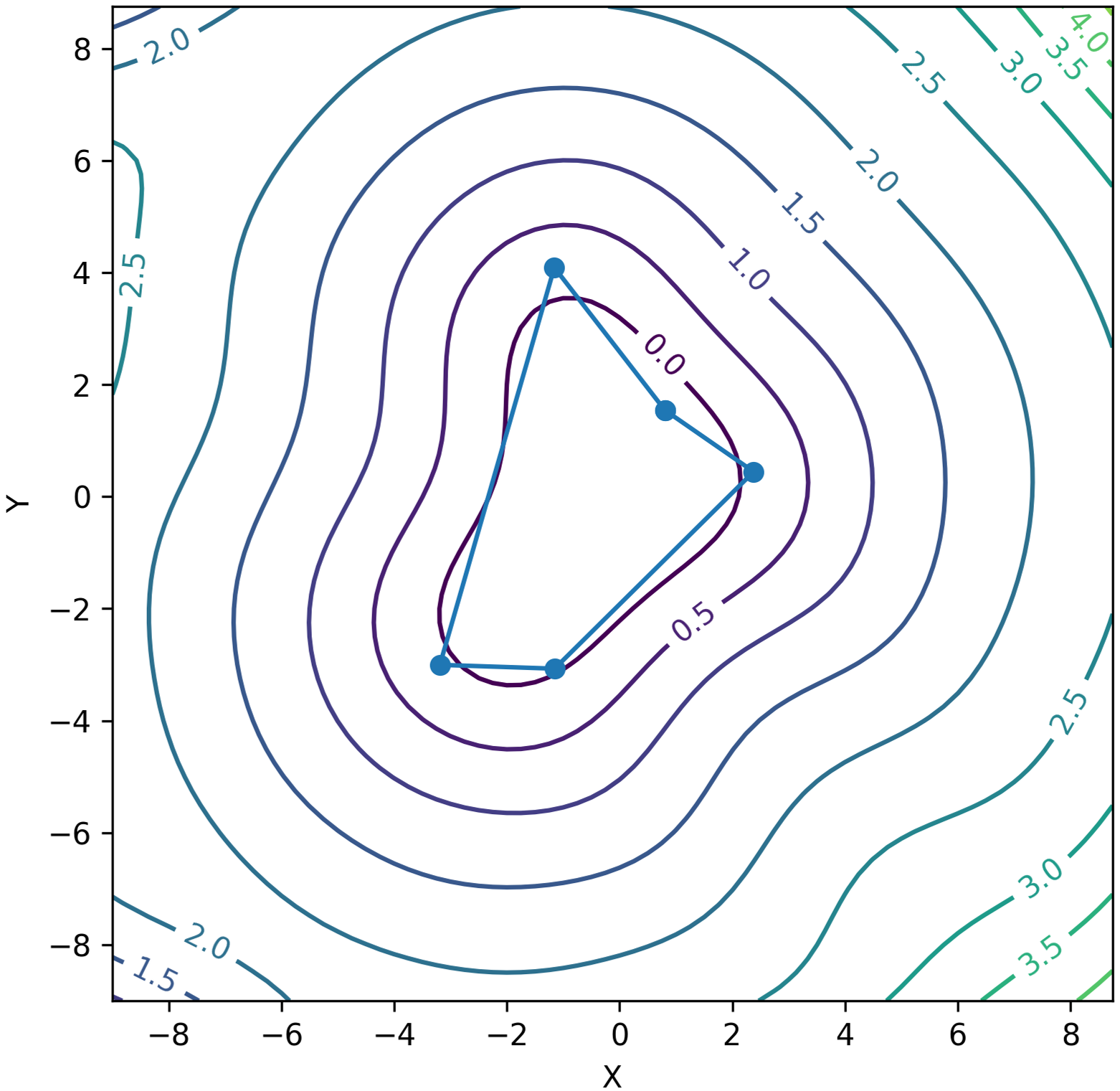}
    \caption{$V(x)$}
  \end{subfigure}
  \begin{subfigure}[b]{0.35\textwidth}
    \centering
    \includegraphics[width=\linewidth]{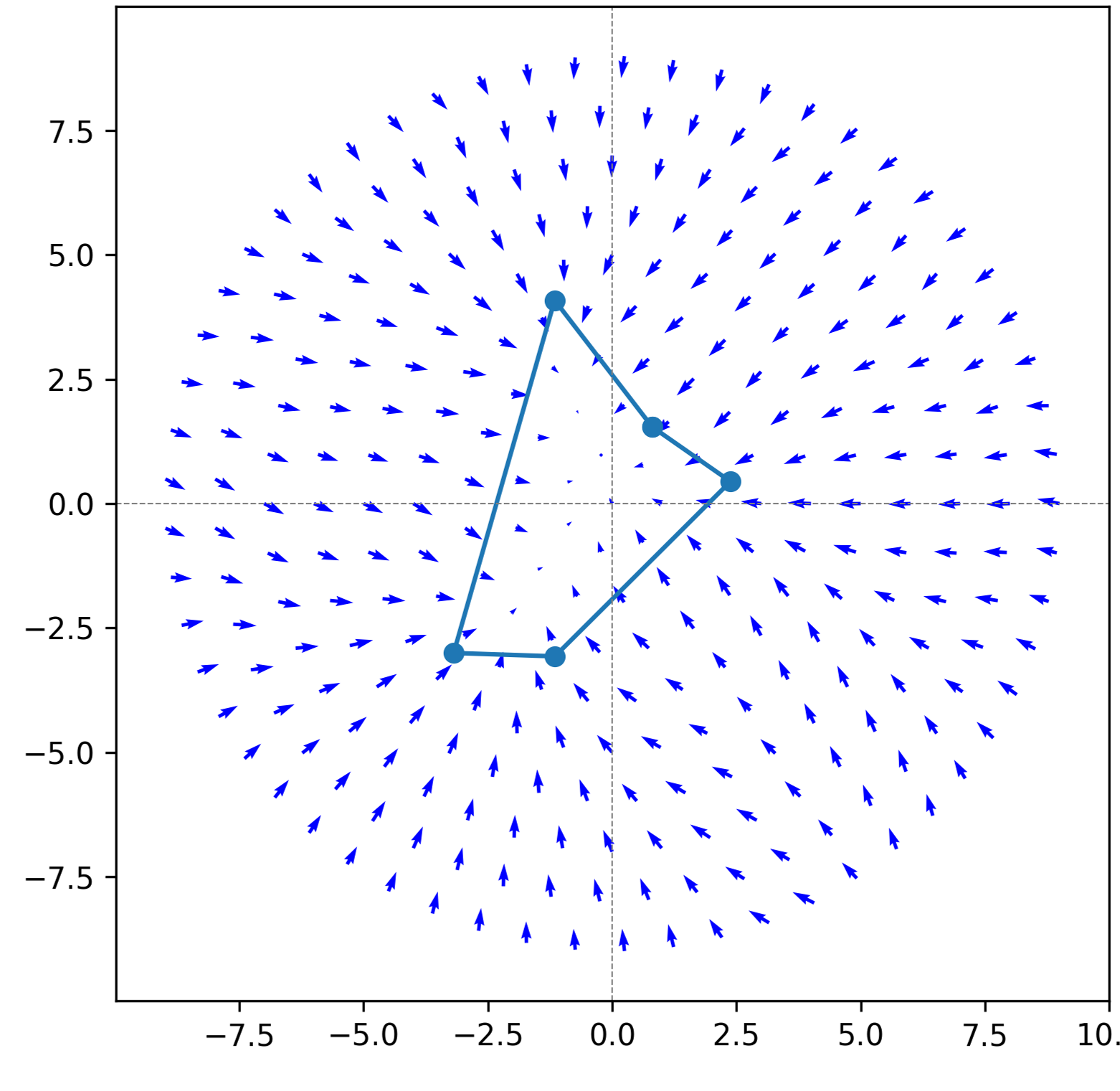}
    \caption{$u^*(x)$}
  \end{subfigure}
  \caption{GP training results for a stationary random polygon
  with $20000$ epochs and maximum $100$ steps for each epoch: (a) shows the contours of the learned value function and (b) shows the learned control vector field.}
  \label{fig:random_target}
\end{figure}

\paragraph{Moving elements.}
We \emph{deliberately omitted} the relative velocity (and any element velocity) from the learner's state.
In the simulation, the agent was controlled using a bounded planar input $u\in\mathbb{R}^2$ (with $\|u\|\le u_{\max}$), and the relative dynamics take the generic form of an input-affine system Eq.~\ref{eq:Input-Affine_System} (constant velocity motion), which however remains unknown to the learning algorithms. 

The goal of these results is to determine whether the effect of the hidden motion can be captured in the learned value function and policy.

These plots allow us to directly inspect how the learned critic and actor encode the hidden motion of an element.

The results for the moving elements are shown in  Fig.~\ref{fig:moving_target}. 

\begin{figure}[htbp]
  \centering
  \begin{subfigure}[b]{0.35\textwidth}
    \centering
    \includegraphics[width=\linewidth]{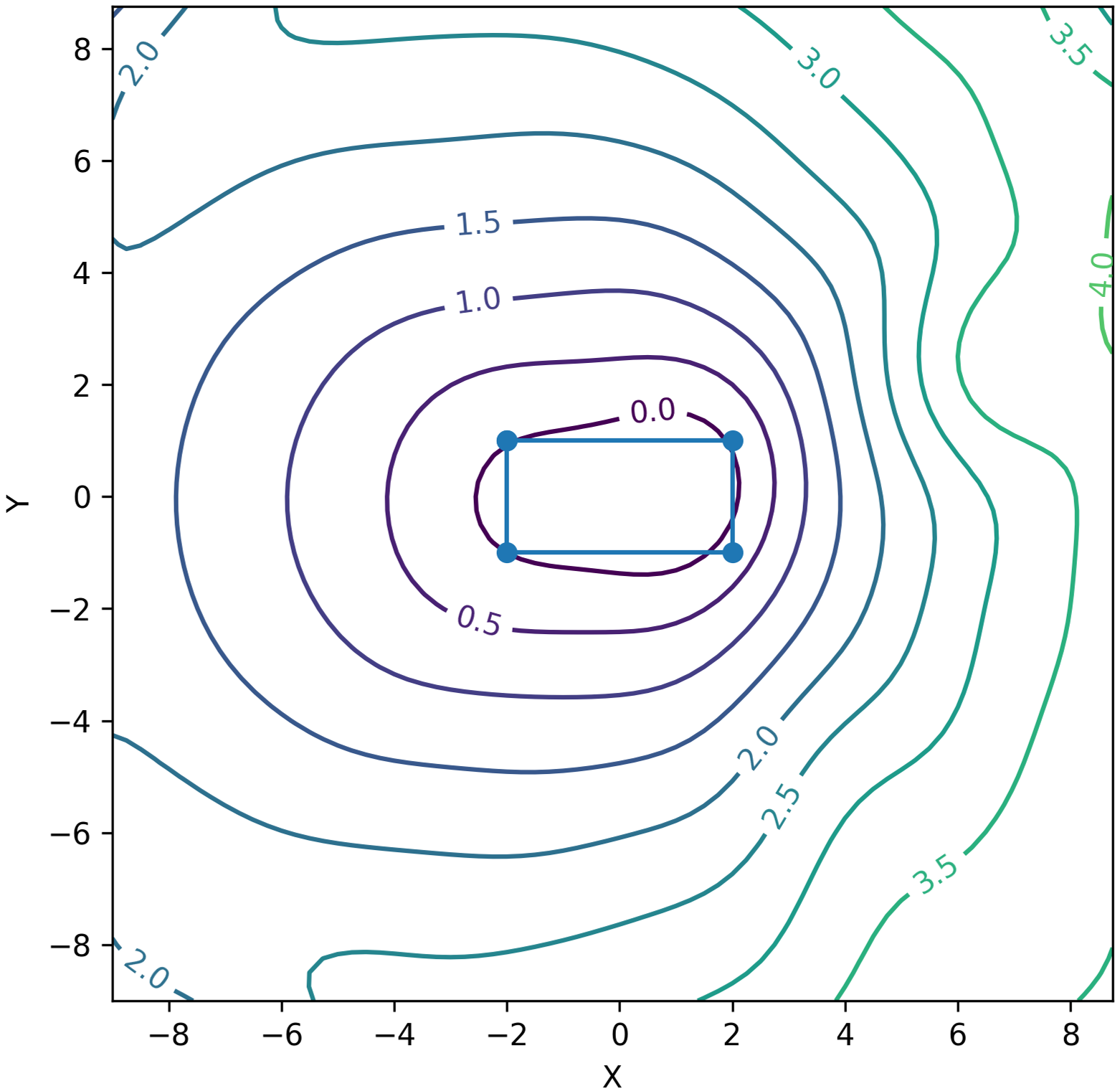}
    \caption{$V(x)$}
    \label{fig:moving_target_V}
  \end{subfigure}
  \begin{subfigure}[b]{0.35\textwidth}
    \centering
    \includegraphics[width=\linewidth]{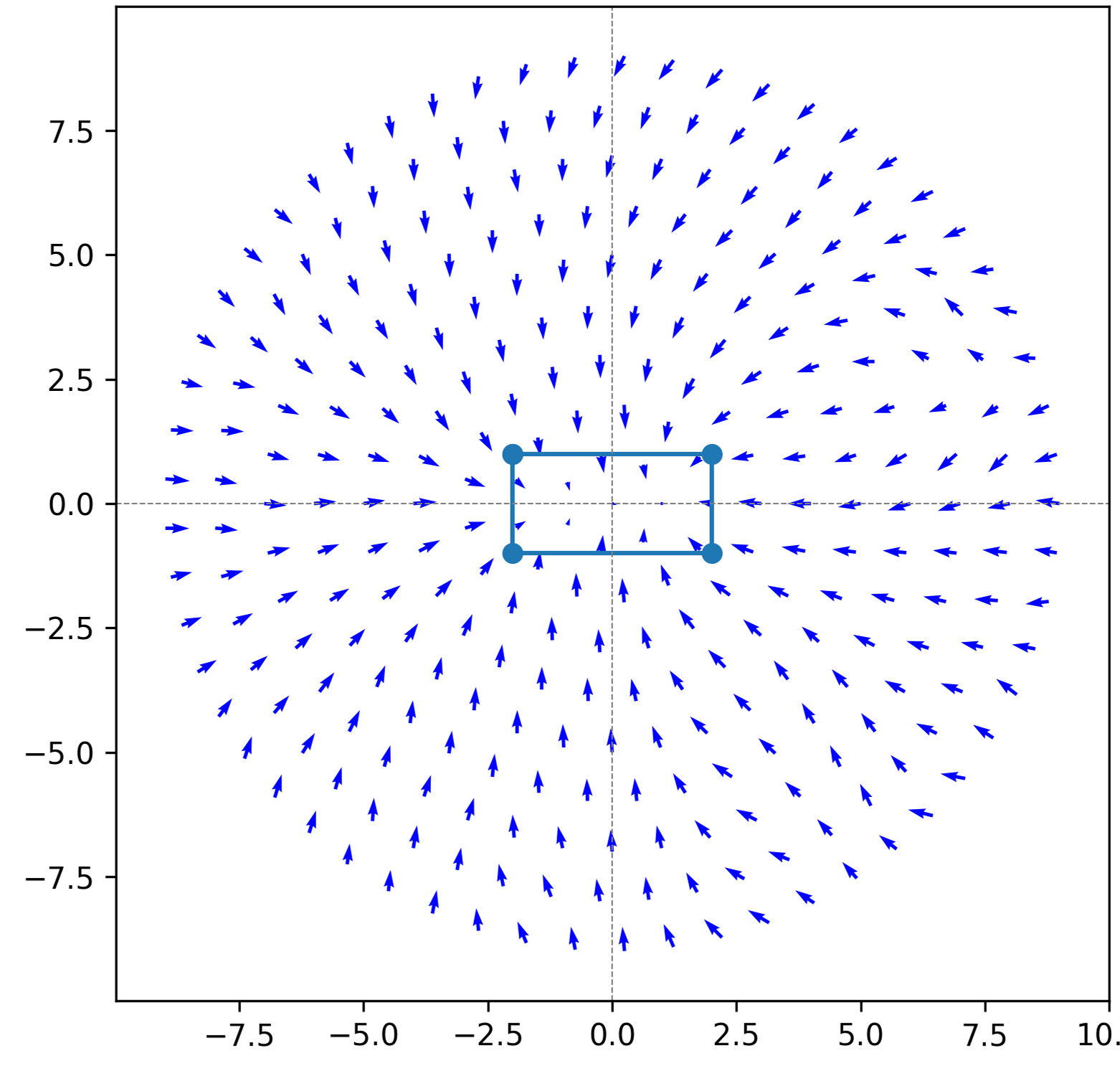}
    \caption{$u^*(x)$}
  \end{subfigure}
  \caption{GP training results for a $4\times 2$ rectangle moving at a speed of $0.8$ along the positive $x$ axis
  with $20000$ epochs and maximum $100$ steps for each epoch: (a) shows the contours of the learned value function and (b) shows the learned control vector field.}
  \label{fig:moving_target}
\end{figure}

Intuitively, when $e$ is moving, the hidden drift must be compensated by the policy, and this compensation appears as systematic asymmetries in the learned fields (e.g., stronger repulsive behavior on the ``upstream'' side of a moving element).

As shown in Fig.~\ref{fig:moving_target_V} that the value function clearly shows that the area in front of the moving element has a significantly lower cumulative cost, indicating that this area is much more dangerous (closer to the element) than the back of the element when serving as an obstacle. However, unlike the PPO critic learning results shown in Fig.~\ref{fig:PPO_moving_result_V}, the boundary where $V_j(x)=0$ of the value function learned by our GP model extends beyond the true boundary of the element, whereas the PPO boundary adheres closely to the true boundary of the element.

\begin{figure}[htbp]
  \centering
  \begin{subfigure}[t]{0.35\textwidth}
    \centering
    \includegraphics[width=\linewidth]{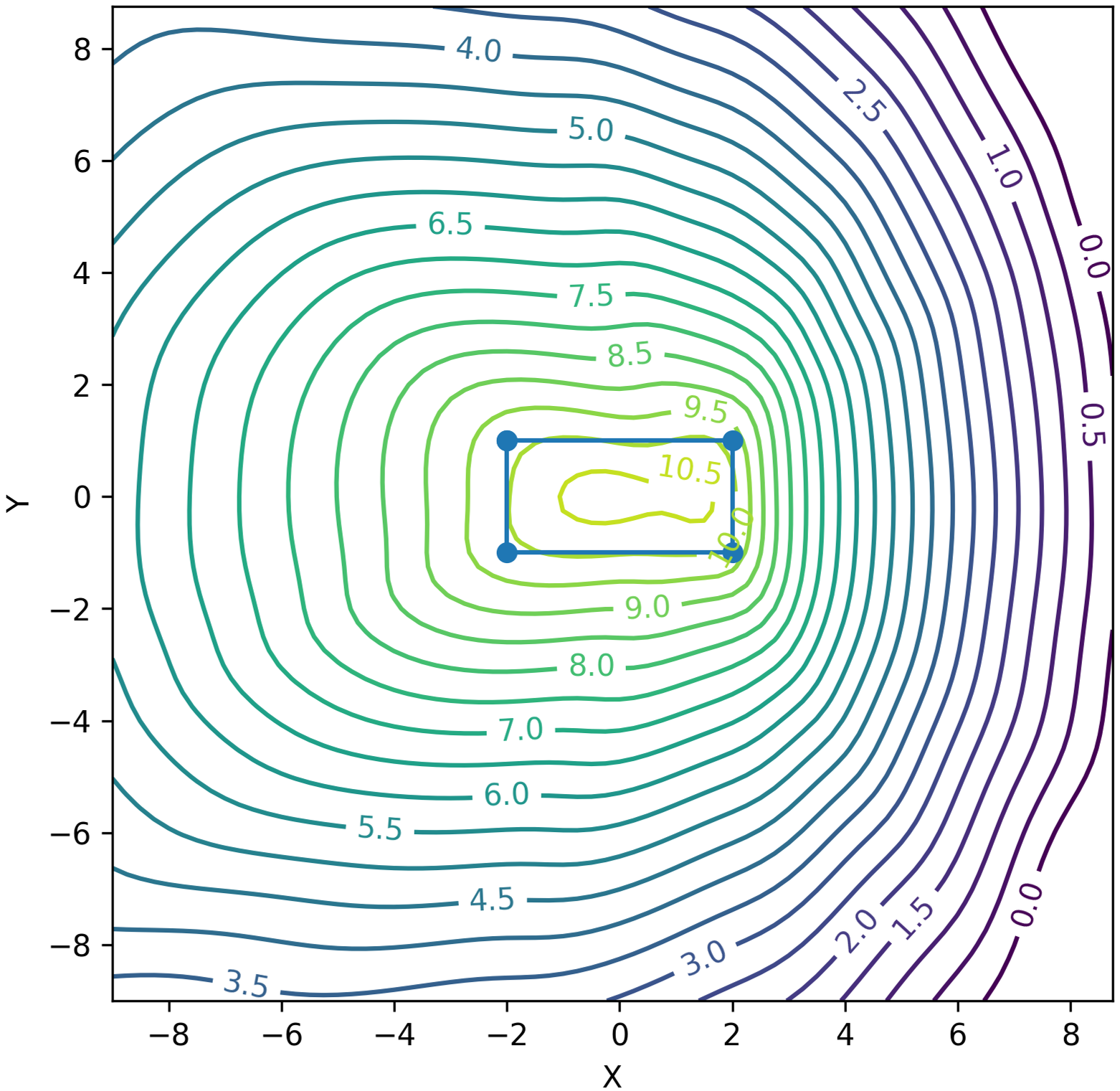}
    \caption{$V_\phi(x)$}
    \label{fig:PPO_moving_result_V}
  \end{subfigure}
  \begin{subfigure}[t]{0.35\textwidth}
    \centering
    \includegraphics[width=\linewidth]{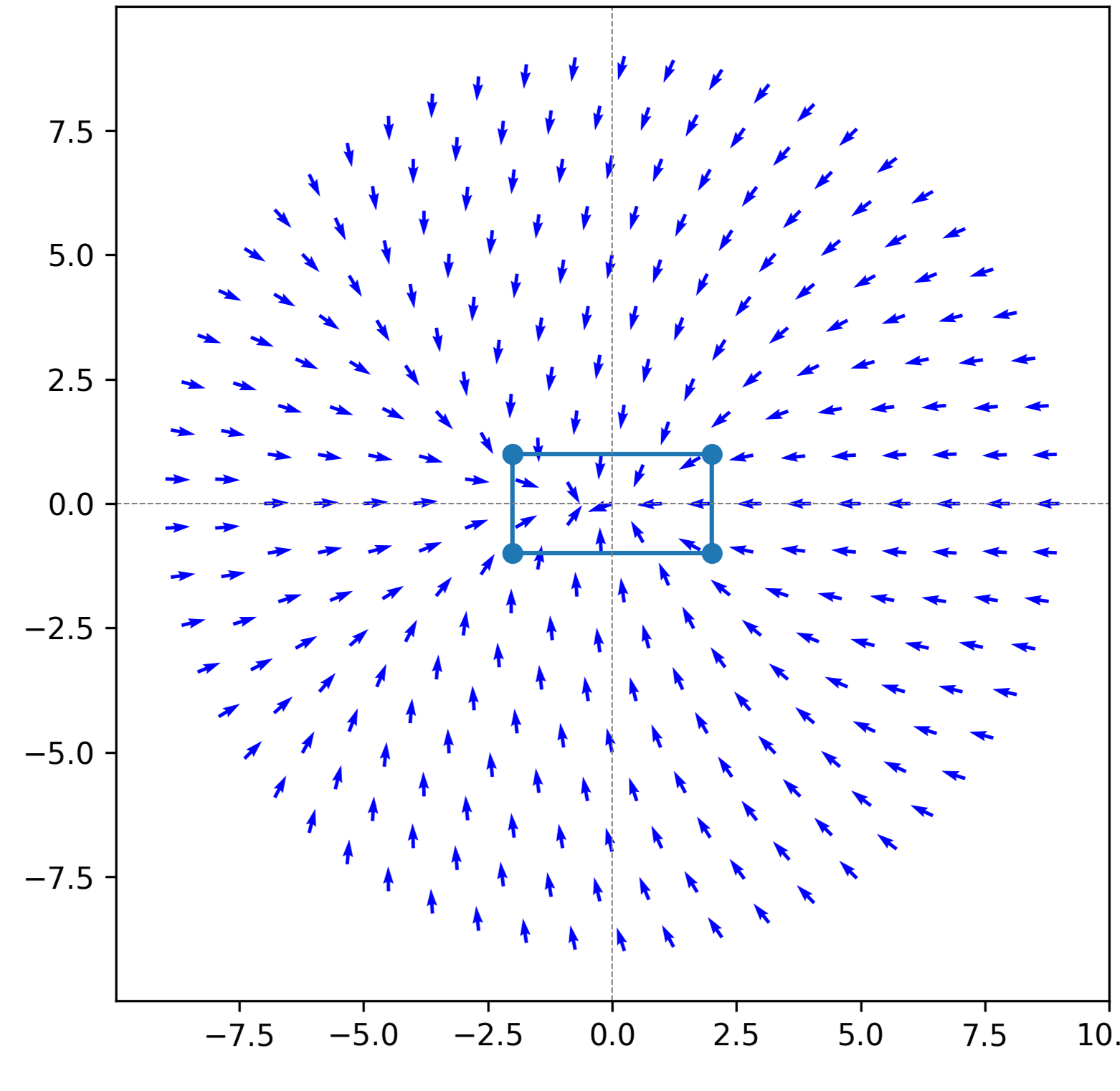}
    \caption{$u_{\theta}(x)$}
    \label{fig:PPO_moving_result_u}
  \end{subfigure}
  \caption{PPO training results for a $4\times 2$ rectangle moving at a speed of $0.8$ along the positive $x$ axis
  with $20000$ epochs and a maximum of $100$ steps for each epoch: (a) shows the contours of the learned value function and (b) shows the mean policy control vector field.}
  \label{fig:PPO_moving_result}
\end{figure}

\section{Modular Composition with Multiple Elements}
\label{sec:modular_deployment}
In this section we propose a method to integrate multiple learned models into a single control system with collision-avoidance and goal-reaching behaviors via a QCQP.

Recall from Sec.~\ref {sec:element-wise_training}, for an element $e_j\in\mathcal{E}$, the value function and the optimal policy estimation are denoted as $\hat{V}_j(x)$ and $\hat{u}^*_j(x)$. For the remainder of this section, we omit the hat from the estimates to simplify notation.

\subsection{CLF-CBF-based Online Composition}
\label{sec:clf-cbf-based-composition}
We propose the use of the value function $V_j$ as the CBF function $H(x)$ in the CBF constraint Eq.~\ref{eq:CBF constaint}:
\begin{equation}
  \dot{V_j}+c_j V_j=\nabla_xV_j^\intercal\dot{x}+c_j V_j \geq 0
\end{equation}

In practice, it is useful to consider a variation of the constraint above that uses a \emph{sharpened} version ${V_j}^q$ of the CBF (to avoid problems where $\nabla V_j$ vanishes near $V_j=0$, $0 <q <1$) and a $V_{\min}$ level set (i.e., $V^q\geq V_{\min}$ to add a safety buffer around the obstacle), yielding
\begin{equation}
\label{eq:CBF constaint mod}
  \dot{{V_j}^q}+c_j ({V_j}^q-V_{\min})=q{V_j}^{(q-1)}\nabla_xV_j^\intercal\dot{x}+c_j(V^q_j-V_{\min}) \geq 0;
\end{equation}

Note that $\dot{V_j}=\nabla_xV_j^\intercal\dot{x}$ is a function affine in action $u$; in our setting, this expression cannot be evaluated directly, but can be obtained from our actor-critic model.
Specifically, we obtain $\nabla_xV_j^\intercal \dot{x}$ from Eq.~\ref{eq:A critic} and substitute the approximation $\hat{A}_{critic}=\hat{A}_{actor}$ to obtain:
\begin{equation}
\begin{split}
\label{eq:V dot}
  \nabla_x V(x)^\intercal \dot{x}=&\hat{A}_{critic}(x,u)-\frac{1}{2}u^\intercal u+\lambda V(x)\\
  =&\frac{1}{2}\bigl(u-u^*(x)\bigr)^\intercal \bigl(u-u^*(x)\bigr)-\frac{1}{2}u^\intercal u+\lambda V(x)\\
  =&\frac{1}{2}\norm{u^*(x)}^2- {u^*(x)}^\intercal u+\lambda V(x);
\end{split}
\end{equation}
As expected, this expression is affine in $u$.

At runtime, we evaluate each element-specific model in the current relative state to obtain $V_{j}(x_{j}), u_{j}^*(x_{j})$. The goal is to choose action $u$ such that Eq.~\ref{eq:CBF constaint} is always satisfied, thus guaranteeing safety (i.e., $V_j\geq 0$). We then solve the following small QCQP that plays the same \emph{architectural role} as CLF-CBF QPs: it selects a control input that aims for goal-seeking behavior while enforcing per-obstacle safety constraints, and a maximum speed $u_{max}$:
\begin{subequations}
\label{eq:qp_controller_modular}
\begin{align}
  u^{\mathrm{QCQP}} &\arg\min_{u}\;\|u-u_g\|^2\\
    \subjectto
  & qV_j^{(q-1)} \left.u_j^\ast\right.^{\intercal} u \leq c_j(V^q_j-V_{\min}) - qV_j^{(q-1)}\left( \frac{1}{2}\norm{u_j^*}^2+\lambda V(x) \right) \nonumber\\
    &\qquad j\in\{1,\dots,M\},\\
    &\norm{u}\le u_{\max}.
\end{align}
\end{subequations}

Intuitively, each obstacle limits the component of the action along the learned approach-to-contact direction, thereby preventing the controller from choosing actions that would aggressively decrease the corresponding $V_j$.

\begin{remark}[Relationship with CLF-CBF QPs]
  Classical CLF-CBF QPs enforce constraints on Lie derivatives (and therefore require the drift and control vector fields $f(x),g(x)$) to guarantee stability/safety properties. In contrast, Eq.~\ref{eq:qp_controller_modular} is \emph{model-free}: it uses only learned value functions and direction fields, yet it retains the key scalability property of optimization-based safety filters---adding an obstacle corresponds to adding one inequality constraint.
\end{remark}

\subsection{Simulation Results after Composing Multiple Moving Obstacles and a Static Goal}
\label{sec:Simulation_Results}

In this section, we present the training results in 2-D simulations to show the training results and collision avoidance via real-time composition.

\paragraph{Simulation Setup.}
\label{sec:sim_setup}

Finally, to demonstrate compositionality, we created a library of actor-critic pairs trained on \emph{isolated} elements (steady and moving). 
For each element configuration (shape and motion profile), we train an element-specific primitive consisting of a value function $V_{j}(x_{j})$ and an optimal policy $u_{j}^*(x_{j})$, using the same actor--critic procedure for both goals and obstacles. 

We then deploy them in a \emph{previously unseen} multi-element scenario: a road crossing setting with multiple moving vehicles and a goal region. It is a simulation of a street-crossing scenario to demonstrate the capabilities of our system. The red rectangles represent moving cars. The blue square represents the goal area. The green dot represents the agent.

The controller is synthesized online by composition and without any additional training, and the goal and obstacle are combined online by the QCQP controller of Sec.~\ref{sec:modular_deployment}, which mediates goal-seeking and collision avoidance in real time.

\begin{figure}[htbp]
  \centering
  \begin{subfigure}[b]{0.26\textwidth}
    \centering
    \includegraphics[width=\linewidth]{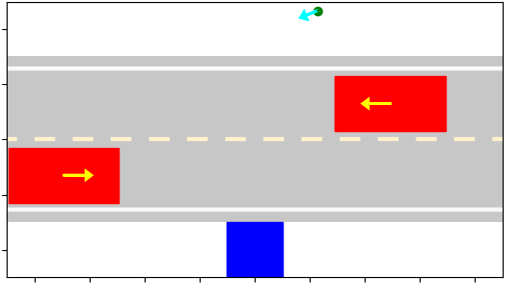}
    \caption{Step $0$}
  \end{subfigure}\hfill
  \begin{subfigure}[b]{0.26\textwidth}
    \centering
    \includegraphics[width=\linewidth]{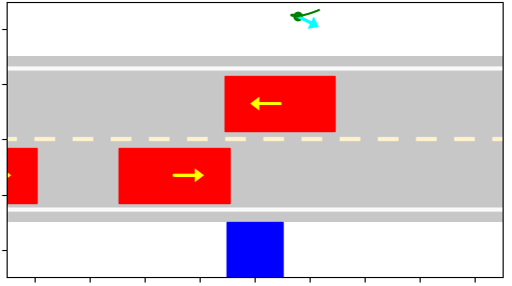}
    \caption{Step $50$}
  \end{subfigure}\hfill
  \begin{subfigure}[b]{0.26\textwidth}
    \centering
    \includegraphics[width=\linewidth]{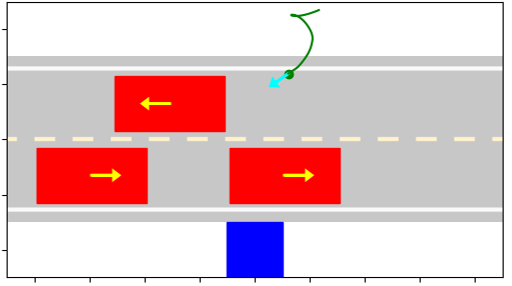}
    \caption{Step $100$}
  \end{subfigure}\hfill
  \begin{subfigure}[b]{0.26\textwidth}
    \centering
    \includegraphics[width=\linewidth]{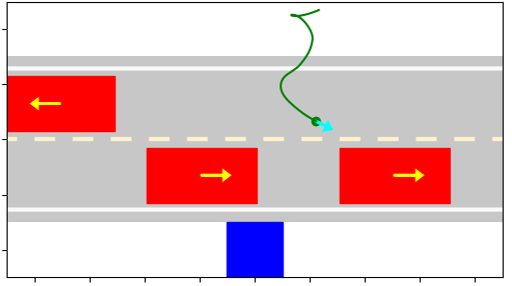}
    \caption{Step $150$}
  \end{subfigure}\hfill
  \begin{subfigure}[b]{0.26\textwidth}
    \centering
    \includegraphics[width=\linewidth]{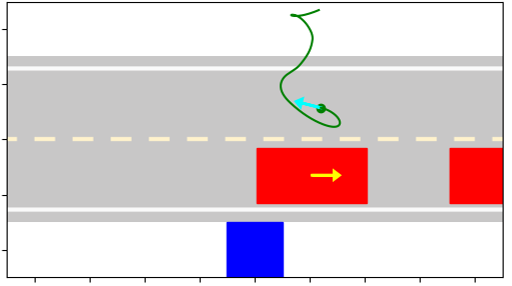}
    \caption{Step $200$}
  \end{subfigure}\hfill
  \begin{subfigure}[b]{0.26\textwidth}
    \centering
    \includegraphics[width=\linewidth]{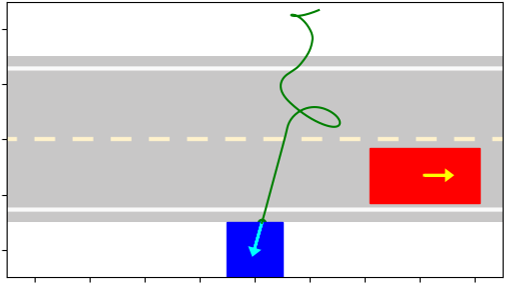}
    \caption{Step $251$}
  \end{subfigure}
  \caption{GP simulation results for a street crossing scenario with two-way traffic.}
  \label{fig:GPL_cross_street}
\end{figure}

\paragraph{Simulation results.}
The simulation results obtained by composing the models trained via the GP are shown in Fig.~\ref{fig:GPL_cross_street} 

At the beginning of the simulation, the agent only sensed the car approaching from the right. However, if the agent attempts to overtake the car and cross the street, it deviates substantially from the target direction. Therefore, it stopped at the position closest to the goal, waited for the car to pass, and continued crossing the street. The agent performs a waiting-and-pass cycle for cars traveling in the opposite directions. After all cars have passed, the agent adheres to the goal-reaching controls without interference.


We can apply the same composition strategies to the training obtained via PPO. The results are shown in Fig.~\ref{fig:PPO_cross_street} where the PPO-trained functions are tested in the same scenario as shown in Fig.~\ref{fig:GPL_cross_street}. 
Although the actions selected by the two systems are not identical, both methods exhibit a similar high-level behavior, alternating between periods of dwelling in place and moving toward the goal.

\begin{figure*}[htbp]
  \centering
  \begin{subfigure}[t]{0.26\textwidth}
    \centering
    \includegraphics[width=\linewidth]{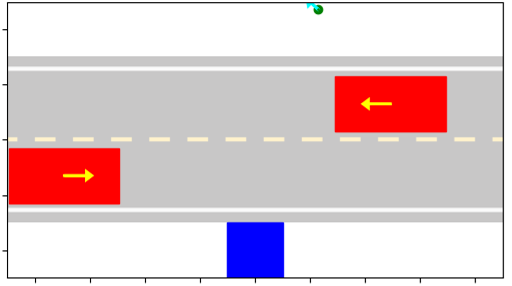}
    \caption{PPO: Step $0$}
  \end{subfigure}\hfill
  \begin{subfigure}[t]{0.26\textwidth}
    \centering
    \includegraphics[width=\linewidth]{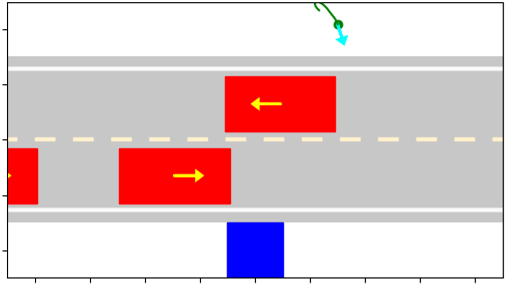}
    \caption{PPO: Step $50$}
  \end{subfigure}\hfill
  \begin{subfigure}[t]{0.26\textwidth}
    \centering
    \includegraphics[width=\linewidth]{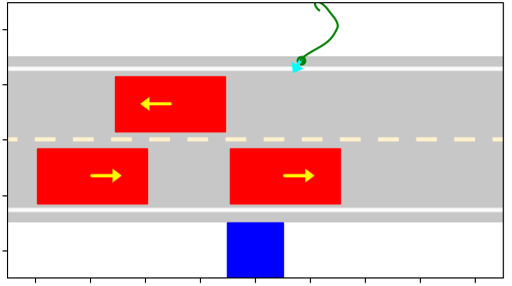}
    \caption{PPO: Step $100$}
  \end{subfigure}\hfill
  \begin{subfigure}[t]{0.26\textwidth}
    \centering
    \includegraphics[width=\linewidth]{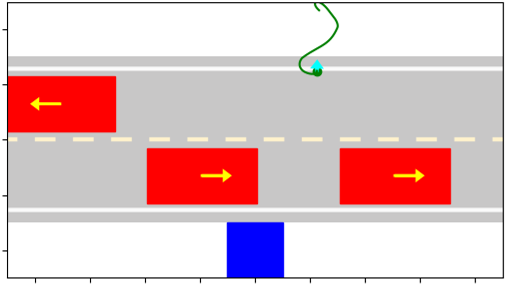}
    \caption{PPO: Step $150$}
  \end{subfigure}\hfill
  \begin{subfigure}[t]{0.26\textwidth}
    \centering
    \includegraphics[width=\linewidth]{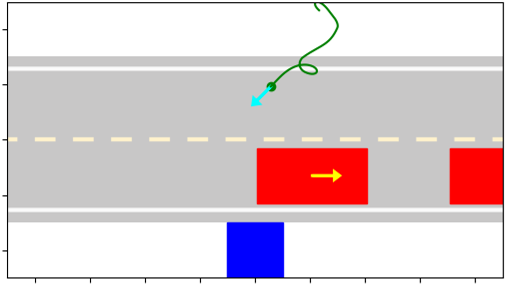}
    \caption{PPO: Step $200$}
  \end{subfigure}\hfill
  \begin{subfigure}[t]{0.26\textwidth}
    \centering
    \includegraphics[width=\linewidth]{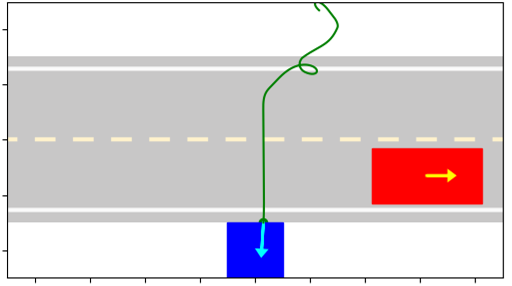}
    \caption{PPO: Step $252$}
  \end{subfigure}
  \caption{PPO simulation results for a street crossing scenario with two-way traffic.}
  \label{fig:PPO_cross_street}
\end{figure*}


\section{Discussion}
As shown in the simulations,  
after training, our method and PPO recovered qualitatively similar structures in both the value function $V(x)$ and the induced optimal control field $u^*(x)$. This suggests that in our setting, our method and PPO can achieve similar training outcomes.


We draw two conclusions:
\begin{enumerate*}
    \item First, although our framework is derived from a continuous-time HJB formulation, PPO remains a useful baseline because it can be trained on a discrete-time approximation of the same environment and objective, and its learned policy can be executed in the same online QP-based controller used during deployment. However, the PPO does not directly optimize our continuous-time HJB residual; instead, it optimizes a discrete-time surrogate objective based on trajectory rollouts and advantage estimates. Bridging these views (e.g., developing policy-optimization updates that are native to continuous time) is an interesting direction for future work.
    \item Second, our modular and composable architecture is not tied to a specific learner: any method that provides a critic estimating a value function and an actor providing a state-dependent policy can be integrated into the same online QP to balance goal-reaching and safety in environments more complex than those encountered during training.
\end{enumerate*}

\section{Conclusion and Future Work}

We introduced a composable, model-free reinforcement learning framework for collision-avoidance navigation in continuous time for input-affine systems. The key idea is to train an \emph{independent} actor-critic (a value function and corresponding policy field for each goal or obstacle element) using a continuous-time discounted HJB residual that only requires state trajectories, time derivatives (estimated from data), and termination events. At deployment time, these independently trained actor-critic pairs are combined online through linear constraints in a QCQP that plays the same architectural role as CLF-CBF QPs, seeking optimal goal-seeking behavior while enforcing per-obstacle safety constraints, and they scale gracefully by adding one constraint per obstacle.

Using Gaussian Processes as the framework backbone of our general function-approximation template, we demonstrated that the learned results recover interpretable value functions and optimal control fields for both stationary and moving elements, capturing the hidden drift factor $f(x)$ introduced by element motion even when the learner observes only the relative position. We further showed that composing results trained in isolation enables real-time navigation in previously unseen multi-element scenarios such as a crosswalk with two-way traffic. Finally, we compared against PPO trained on a discrete-time approximation of the same environment and found qualitatively similar learned structures and comparable collision-avoidance behavior under the same QCQP-based deployment controller.

Several directions remain for future research. On the learning side, it is important to study the stability and safety properties of the composed controller and to characterize the effect of approximation errors in learned results on closed-loop behavior. On the modeling side, scaling beyond low-dimensional relative states requires sparse or structured approximations and careful treatment of partial observability when the hidden dynamics are significant. Finally, a major advantage of Gaussian Processes is their ability to quantify uncertainty; we plan to treat the HJB residual as a measurement and propagate uncertainty over the stored base-point parameters via Kalman-filter-style covariance updates. Such calibrated uncertainty estimates can be integrated into the online QCQP (e.g., via risk-sensitive or chance-constrained formulations) to improve robustness in dynamic and uncertain environments.

%
%
%
\bibliographystyle{splncs04}
\bibliography{citation}

\end{document}